\documentclass[runningheads]{llncs}

\usepackage[T1]{fontenc}
\usepackage{graphicx}
\usepackage{amsmath,amssymb}
\usepackage{algorithm}
\usepackage{algorithmic}

\usepackage{booktabs}
\usepackage{multirow}
\usepackage{xcolor}
\usepackage{hyperref}
\usepackage{cleveref}

\newcommand{\R}{\mathbb{R}}
\newcommand{\N}{\mathcal{N}}
\newcommand{\E}{\mathbb{E}}
\newcommand{\G}{\mathcal{G}}
\newcommand{\bW}{\mathbf{W}}
\newcommand{\bx}{\mathbf{x}}
\newcommand{\bh}{\mathbf{h}}

\begin{document}

\title{Cost-Sensitive Neighborhood Aggregation\\for Heterophilous Graphs:\\When Does Per-Edge Routing Help?}

\titlerunning{CSNA: When Does Per-Edge Routing Help?}

\author{Eyal Weiss}

\authorrunning{E. Weiss}

\institute{Robotics Institute, Carnegie Mellon University\\
(work done at the Computer Science Faculty, Technion)\\
\email{eyalw@andrew.cmu.edu}}

\maketitle

\begin{abstract}
Message-passing graph neural networks (GNNs) aggregate the features of neighboring nodes, which can hurt under heterophily, i.e., when connected nodes tend to differ in label.
We distinguish two heterophily regimes, namely \emph{adversarial}, where cross-class edges dilute the class signal and harm classification, and \emph{informative}, where the heterophilous structure itself carries useful signal.
We ask when it helps to process the message of each edge individually (per-edge routing), and when it suffices to treat all edges alike.
To study this question, we introduce Cost-Sensitive Neighborhood Aggregation (CSNA), a GNN layer that computes the distance between the endpoints of each edge in a learned projection.
It uses this distance to weight each message in two channels with independent transformations, for likely same-class and likely cross-class neighbors.
Under a contextual stochastic block model, we show that mean aggregation can reverse the direction of the difference between the class means under heterophily, and that weighting same-class edges sufficiently more than cross-class edges preserves it.
On six benchmarks, under a common tuning protocol, CSNA is within 0.5 percentage points of the most accurate method on two of the four adversarial-heterophily datasets and 4 to 5 points below it on the other two.
On the two informative-heterophily datasets, it is about 13 and 16 points below a standard graph convolutional network (GCN).
The learned distance separates same-class from cross-class edges only weakly (an AUC of 0.52 to 0.73), with no consistent difference between the regimes.
Thus, on these benchmarks, we find no evidence that per-edge routing in this form adds value over uniform channels.
Code is available at \url{https://github.com/eyal-weiss/CSNA-public}.

\end{abstract}

\section{Introduction}\label{sec:intro}

Message-passing Graph Neural Networks (GNNs)~\cite{kipf2017semi,velickovic2018graph,hamilton2017inductive} compute the representation of a node by aggregating the features of its neighbors.
This aggregation relies on an implicit \emph{homophily} assumption, namely that connected nodes tend to share labels or properties.
When the assumption holds, aggregation smooths the representations within a class and yields strong node classification~\cite{wu2020comprehensive}.
However, many real-world graphs exhibit \emph{heterophily}, i.e., connected nodes frequently differ in label~\cite{pei2020geom,zhu2020beyond}.
On such graphs, aggregation can dilute the class signal instead of reinforcing it, and standard GNNs can perform worse than a Multi-Layer Perceptron (MLP) that ignores the graph structure entirely~\cite{zhu2020beyond}.
For example, in our experiments a Graph Convolutional Network (GCN)~\cite{kipf2017semi} is roughly 22 to 33 percentage points (pp) less accurate than an MLP on the three WebKB graphs, namely Texas, Wisconsin, and Cornell (\Cref{tab:results}).

Heterophily is not always harmful~\cite{luan2022revisiting,platonov2023critical}.
We distinguish two regimes, which we call adversarial and informative.
Under \emph{adversarial} heterophily, edges between nodes of different classes (\emph{cross-class} edges) dilute the class signal and harm classification.
Under \emph{informative} heterophily, the heterophilous structure itself carries useful signal.

Several architectures address heterophily (\Cref{sec:related}).
The closest to our work is ACM-GNN~\cite{luan2022revisiting}, which mixes three channels with per-node gating, namely a low-pass channel, which averages over the neighbors, a high-pass channel, which takes differences from the neighbors, and an identity channel.
Each of these channels is a fixed spectral filter and is thus \emph{uniform}, i.e., every edge receives the same treatment within a given channel.
This raises the following question: \emph{When does routing each message individually help, and when is a uniform spectral channel sufficient?}

To study this question, we propose \emph{Cost-Sensitive Neighborhood Aggregation} (CSNA).
For each edge, CSNA computes the distance between the endpoints in a learned projection space, which we call the \emph{cost} of the edge (hence the name of the method).
It uses the cost to \emph{soft-route} the message sent along the edge, i.e., to weight the message in two channels, each with its own learned transformation.
The \emph{concordant} channel emphasizes low-cost messages (likely same-class), and the \emph{discordant} channel emphasizes high-cost messages (likely cross-class).
A per-node gating mechanism then combines the two channels with the node's own (\emph{ego}) representation.
Thus the routing of CSNA is \emph{per-edge}, i.e., each edge receives an individual routing weight, in contrast to the uniform channels of ACM-GNN.
This finer granularity requires additional per-edge computation, which makes training $3$--$9\times$ slower than for GCN (\Cref{app:runtime}).

\paragraph{The Hypothesis Behind CSNA.}
\Cref{fig:toy} illustrates the scenario in which we expected per-edge routing to outperform uniform spectral channels.
In a heterophilous graph, the cross-class edges may have \emph{mixed utility}, i.e., the transformed features of some cross-class neighbors carry complementary information that aids classification, while those of others are misleading.
A uniform high-pass channel treats all cross-class edges identically.
If the learned cost separates the two kinds of edges, CSNA can upweight the informative ones and downweight the misleading ones.
We expected this to matter in the adversarial-heterophily regime and not in the informative-heterophily regime, where we expected nearly all cross-class edges to carry useful signal.
\Cref{sec:experiments} tests a simpler premise of this hypothesis, namely that the learned cost carries information about the labels at the two ends of an edge, and finds only weak support for it.

\begin{figure}[t]
\centering
\includegraphics[width=0.95\textwidth]{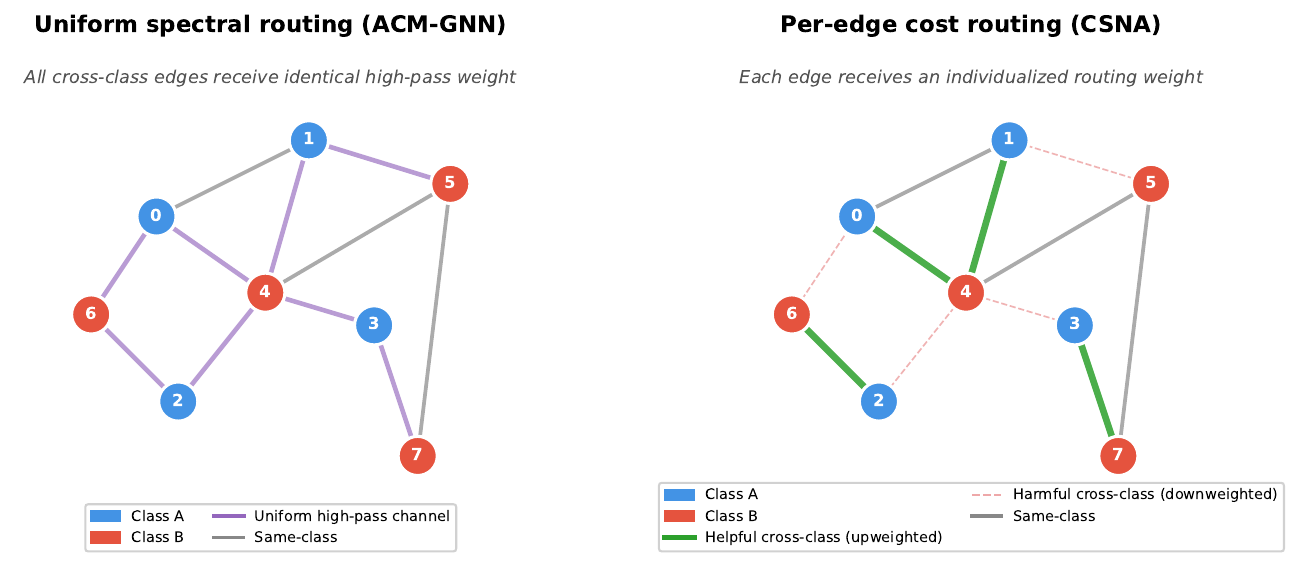}
\vspace{-2mm}
\caption{The hypothesis behind CSNA on a toy heterophilous graph (illustrative, not an experimental result). Node colors mark the two classes, and same-class edges are gray. Left: A uniform high-pass channel assigns the same weight to all cross-class edges (purple). Right: Per-edge cost routing could upweight helpful cross-class edges (thick green) and downweight harmful ones (thin dashed red), provided that the cost function separates them.}\label{fig:toy}
\end{figure}

\paragraph{Contributions.}
(i)~We give an empirical answer to this question on six benchmarks, and it is largely negative (\Cref{sec:experiments}).
CSNA exceeds an MLP by at most 0.7pp on the adversarial-heterophily datasets and is about 13 and 16pp below GCN on the two informative-heterophily datasets.
Its learned cost function separates same-class from cross-class edges (the two \emph{edge types}) only weakly, and the separation does not distinguish the two regimes.
(ii)~We introduce CSNA, a dual-channel message-passing layer with per-edge cost-based routing (\Cref{sec:method}).
Its default variant, called \emph{lite}, uses only the distance between the endpoints in a learned projection as the cost, and an \emph{extended} variant adds a second cost component, computed by a learned function of the two endpoints.
(iii)~We analyze a \emph{contextual stochastic block model} (CSBM), i.e., a random graph in which the edge probabilities and the node features depend on the classes of the nodes.
We prove that under heterophily mean aggregation can reverse the direction of the class signal (the difference between the class means).
We also prove that a weighted aggregation whose weights are independent of the features preserves this direction if and only if the ratio between the mean weights of same-class and cross-class edges exceeds the ratio between the cross-class and same-class edge probabilities (\Cref{sec:theory}).

\section{Related Work}\label{sec:related}

We review GNN architectures designed for heterophily, compare CSNA with the closest of them, and then relate CSNA to two known pathologies of deep GNNs.

\paragraph{Heterophily-Aware GNNs.}
Standard GNNs, such as GCN~\cite{kipf2017semi}, GAT~\cite{velickovic2018graph}, and GraphSAGE~\cite{hamilton2017inductive}, degrade as homophily decreases~\cite{zhu2020beyond}.
Among the architectures that address this, we distinguish three families.
The first separates the ego representation from the embeddings of higher-order neighbors (H2GCN~\cite{zhu2020beyond}).
The second uses spectral filters that can have both low-pass and high-pass responses, either by learning the coefficients of a polynomial graph filter (GPRGNN~\cite{chien2021adaptive}) or by mixing fixed filters with per-node gating (ACM-GNN~\cite{luan2022revisiting}).
The third assigns positive or negative attention weights to edges (FAGCN~\cite{bo2021beyond}).
The methods of the first two families apply the same operation to all edges of a node, and those of the third weight each edge individually within a single channel.
In contrast, CSNA routes each message between two channels with separate transformations, according to a learned cost of its edge.

\paragraph{Comparison with ACM-GNN.}
ACM-GNN~\cite{luan2022revisiting} is the closest method to CSNA.
The two share the architecture of aggregation channels combined by per-node gating, and they differ in how the channels are constructed.
The channels of ACM-GNN are defined by fixed spectral operations (mean aggregation for low-pass and signed aggregation for high-pass), so every edge within a channel is treated identically, whereas CSNA routes each message by the learned cost of its edge.
This is finer-grained but more expensive (see the complexity paragraph in \Cref{sec:csna}).

\paragraph{Over-Smoothing and Over-Squashing.}
Deep GNNs suffer from \emph{over-smoothing}, in which the representations of different nodes become indistinguishable as layers are added~\cite{li2018deeper,oono2020graph}, and from \emph{over-squashing}, in which information from many distant nodes is compressed into fixed-size vectors~\cite{alon2021on,topping2022understanding}.
These pathologies are compounded under heterophily.
Cost-based routing is meant to reduce cross-class smoothing, and it does not address over-squashing, which is a topological bottleneck issue.

\section{Method}\label{sec:method}

This section defines the CSNA layer, its extended variant, its training objective, and its complexity.

\subsection{Preliminaries}

Let $\G = (V, E)$ be an undirected graph with $n = |V|$ nodes and feature matrix $\mathbf{X} \in \R^{n \times d}$.
Each node $i$ has a label $y_i \in \{1, \ldots, C\}$.
Recall that an edge $(i,j)$ is cross-class if $y_i \neq y_j$, and we say that it is same-class otherwise.
The \emph{edge homophily ratio} is $\mathcal{H} = |\{(i,j) \in E : y_i = y_j\}|/|E|$, i.e., the fraction of same-class edges.

In standard message-passing~\cite{gilmer2017neural}, the representation of node $i$ at layer $\ell+1$ is
\begin{equation}\label{eq:mp}
    \bh_i^{(\ell+1)} = \sigma\!\Bigl(\sum_{j \in \N(i)} \alpha_{ij}\, \bW^{(\ell)}\, \bh_j^{(\ell)}\Bigr),
\end{equation}
where $\N(i)$ is the set of neighbors of node $i$, which includes $i$ itself, $\alpha_{ij}$ are aggregation weights, and $\sigma$ is a nonlinearity.

\subsection{Cost-Sensitive Neighborhood Aggregation}\label{sec:csna}

A CSNA layer assigns a cost to each edge and maps it to a \emph{concordance score}, which is high when the cost is low.
Normalizing the scores over the neighbors of a node gives the \emph{routing weights} of the concordant channel, and normalizing their complements gives those of the discordant channel.
A third channel, the \emph{ego channel}, transforms the representation of the node itself, and a \emph{gate} computes, for each node, the weights that combine the three channels.
We describe these four steps in turn, with the corresponding lines of \Cref{alg:csna}.

\paragraph{Step 1: Edge Cost.}
For each edge $(i,j)$ at layer $\ell$, we compute the distance between its endpoints in a learned projection (Line~2), namely
\begin{equation}\label{eq:g}
    g_{ij}^{(\ell)} = \|\bW_g\, \bh_i^{(\ell)} - \bW_g\, \bh_j^{(\ell)}\|_2,
\end{equation}
where $\bW_g \in \R^{d' \times d}$ is a learned projection matrix.
A low $g_{ij}$ indicates that the endpoints are similar in the projected space (likely same-class), and a high $g_{ij}$ indicates that they diverge (likely cross-class).
The distance $g_{ij}$ is the only cost component in the lite variant of CSNA, which is the default.

\paragraph{Step 2: Concordance Score.}
We convert the cost to a soft concordance score (Line~3),
\begin{equation}\label{eq:concordance}
    s_{ij}^{(\ell)} = \sigma\!\left(\frac{-g_{ij}^{(\ell)}}{\tau}\right) \in (0, 1/2],
\end{equation}
where $\tau > 0$ is a temperature parameter and $\sigma$ denotes, from here on, the sigmoid function.
The score is at most $1/2$ because the cost is nonnegative, and it equals $1/2$ exactly when the cost is zero.

\paragraph{Step 3: Dual-Channel Aggregation.}
We aggregate the messages in the two channels, each with its own transformation (Lines~6--7),
\begin{align}
    \bh_i^{\text{con}} &= \sum_{j \in \N(i)} \tilde{s}_{ij}\, \bW_{\text{con}}\, \bh_j^{(\ell)}, \label{eq:con}\\
    \bh_i^{\text{dis}} &= \sum_{j \in \N(i)} \tilde{d}_{ij}\, \bW_{\text{dis}}\, \bh_j^{(\ell)}, \label{eq:dis}
\end{align}
where $\tilde{s}_{ij} = \mathrm{softmax}_{j \in \N(i)}(s_{ij})$ and $\tilde{d}_{ij} = \mathrm{softmax}_{j \in \N(i)}(1 - s_{ij})$ are the routing weights, and $\bW_{\text{con}}, \bW_{\text{dis}} \in \R^{d' \times d}$ are independent weight matrices.
In our implementation, the softmax is taken over the neighbors of the sending node $j$ instead of the receiving node $i$.
We tested both and found no consistent difference (\Cref{app:repro}).

\paragraph{Step 4: Gated Combination.}
The output of the layer combines the two aggregation channels with the ego channel via per-node gating (Lines~8--9),
\begin{equation}\label{eq:gate}
    \bh_i^{(\ell+1)} = \sum_{k \in \{\text{con}, \text{dis}, \text{self}\}} \gamma_k(\bh_i) \cdot \bh_i^{k},
\end{equation}
where $\bh_i^{\text{self}} = \bW_{\text{self}}\, \bh_i^{(\ell)}$ is the ego channel.
The gate is
\[
    \gamma(\bh_i) = \mathrm{softmax}\bigl(\bW_\gamma [\bh_i^{\text{con}} \| \bh_i^{\text{dis}} \| \bh_i^{\text{self}}]\bigr) \in \R^3,
\]
where $[\cdot \| \cdot]$ denotes vector concatenation.

\paragraph{Extended Variant.}
The extended variant adds a second cost component $h_{ij}$ to the distance, namely
\begin{equation}\label{eq:f}
    f_{ij}^{(\ell)} = g_{ij}^{(\ell)} + h_{ij}^{(\ell)}, \quad\text{where}\quad h_{ij}^{(\ell)} = \mathrm{softplus}\bigl(\mathbf{a}^\top [\bW_g \bh_i^{(\ell)} \| \bW_g \bh_j^{(\ell)}]\bigr),
\end{equation}
where $\mathbf{a} \in \R^{2d'}$ is a learnable parameter vector, and $f_{ij}$ replaces $g_{ij}$ in \Cref{eq:concordance}.
In our experiments, the extended variant is more accurate than the lite variant by more than 2pp on one of the six datasets only (\Cref{app:variants}).
We therefore present the lite variant ($g_{ij}$ only) as the default.

\paragraph{Training Objective.}
The lite variant is trained with the cross-entropy classification loss $\mathcal{L}_{\text{CE}}$ only.
The extended variant adds a regularizer on its second cost component, computed on the edges between labeled nodes, namely
\begin{equation}\label{eq:cal}
    \mathcal{L}_{\text{reg}} = \frac{1}{|E_{\text{train}}|} \sum_{(i,j) \in E_{\text{train}}} \mathrm{ReLU}\bigl(h_{ij} - \mathbb{1}[y_i \neq y_j]\bigr),
\end{equation}
where $E_{\text{train}}$ denotes the set of these edges.
That is, the component $h_{ij}$ of a same-class edge is penalized whenever it is positive, and that of a cross-class edge only when it exceeds $1$.
The objective of the extended variant is $\mathcal{L}_{\text{CE}} + \lambda_{\text{reg}}\, \mathcal{L}_{\text{reg}}$ with $\lambda_{\text{reg}} = 0.1$.

\paragraph{Architecture Details.}
We initialize the gate bias to $[0, 0, 1]$, which favors the ego channel, so that the model starts near MLP-like behavior.
\Cref{app:repro} gives the remaining details of the architecture.

\paragraph{Complexity.}
ACM-GNN has the same asymptotic complexity as GCN, namely $O(|E| \cdot d)$ per channel with a fixed structure, while CSNA requires an additional $O(|E| \cdot d)$ distance computation.
In our experiments, this makes training CSNA $3$--$9\times$ slower than training GCN (\Cref{app:runtime}).

\begin{algorithm}[t]
\caption{CSNA Layer (Lite Variant)}\label{alg:csna}
\begin{algorithmic}[1]
\REQUIRE Node features $\mathbf{H}^{(\ell)} \in \R^{n \times d}$, edge set $E$, temperature $\tau$
\ENSURE Updated features $\mathbf{H}^{(\ell+1)} \in \R^{n \times d'}$
\FOR{each edge $(i,j) \in E$}
    \STATE $g_{ij} \leftarrow \|\bW_g \bh_i - \bW_g \bh_j\|_2$ \COMMENT{Pairwise distance in learned projection}
    \STATE $s_{ij} \leftarrow \sigma(-g_{ij}/\tau)$ \COMMENT{Concordance score}
\ENDFOR
\FOR{each node $i \in V$}
    \STATE $\bh_i^{\text{con}} \leftarrow \sum_{j \in \N(i)} \tilde{s}_{ij}\, \bW_{\text{con}} \bh_j$ \COMMENT{Concordant channel}
    \STATE $\bh_i^{\text{dis}} \leftarrow \sum_{j \in \N(i)} \tilde{d}_{ij}\, \bW_{\text{dis}} \bh_j$ \COMMENT{Discordant channel}
    \STATE $\bh_i^{\text{self}} \leftarrow \bW_{\text{self}} \bh_i$ \COMMENT{Ego channel}
    \STATE $\bh_i^{(\ell+1)} \leftarrow \sum_k \gamma_k(\bh_i) \cdot \bh_i^k$ \COMMENT{Gated combination}
\ENDFOR
\end{algorithmic}
\end{algorithm}

\section{Theoretical Analysis}\label{sec:theory}

This section compares cost-sensitive aggregation with mean aggregation on heterophilous graphs drawn from the CSBM~\cite{zhu2020beyond}.
Throughout this section, we assume two classes ($C=2$), for tractability (see \Cref{rem:scope}).
The analysis is not specific to CSNA or to either of its variants.
\Cref{thm:preservation} below concerns weighted aggregation with weights that are independent of the features, and \Cref{rem:applicability} discusses in what sense it bears on CSNA.
We begin with the model.

\begin{definition}[Contextual Stochastic Block Model]\label{def:csbm}
A graph $\G$ with $n$ nodes and $d$-dimensional node features is drawn from $\mathrm{CSBM}(n, 2, p, q, \mu)$, where $2$ is the number of classes, as follows.
The nodes are divided into two equal-sized classes $V_+, V_-$.
Edges are drawn independently, with probability $p$ within classes and probability $q$ between classes.
The features of node $i$ are drawn as $\bx_i \sim \mathcal{N}(\boldsymbol{\mu}_{y_i}, \mathbf{I}_d)$, with $\boldsymbol{\mu}_+ = +\frac{\mu}{2}\mathbf{e}_1$ and $\boldsymbol{\mu}_- = -\frac{\mu}{2}\mathbf{e}_1$, where $\mathbf{e}_1$ is the first standard basis vector of $\R^d$ and the scalar $\mu > 0$ sets the distance between the class means.
The features are independent of the edges.
\end{definition}

In this model, the fraction of same-class edges tends to $p/(p+q)$ as $n$ grows, so we identify the homophily ratio with $\mathcal{H} = p/(p + q)$, and the heterophilous regime is $q > p$, i.e., $\mathcal{H} < 1/2$.
In this model, we formalize the class signal of \Cref{sec:intro} as the difference $\boldsymbol{\mu}_+ - \boldsymbol{\mu}_-$ between the class means.
We say that an aggregation \emph{preserves} the direction of the class signal if the expected difference between the class means of the aggregated representations is, up to a term that vanishes as $n \to \infty$, a positive multiple of it that does not depend on $n$.
If the multiple is negative, the aggregation \emph{reverses} the direction.
The first result shows that mean aggregation reverses the direction of the class signal under heterophily.

\begin{theorem}[Signal distortion under mean aggregation]\label{thm:degradation}

\noindent Let $\G \sim \mathrm{CSBM}(n, 2, p, q, \mu)$ with constant $p, q \in (0,1)$.
After one round of symmetrically normalized aggregation with self-loops, $\mathbf{H}^{(1)} = \tilde{\mathbf{A}} \mathbf{X}$, where $\tilde{\mathbf{A}} = \mathbf{D}^{-1/2}(\mathbf{A}+\mathbf{I})\mathbf{D}^{-1/2}$, $\mathbf{A}$ is the adjacency matrix, and $\mathbf{D}$ is the diagonal degree matrix of $\mathbf{A}+\mathbf{I}$, the expected class-mean representations satisfy
\begin{equation}
    \E[\bar{\bh}_+^{(1)} - \bar{\bh}_-^{(1)}] = \frac{p-q}{p+q}(\boldsymbol{\mu}_+ - \boldsymbol{\mu}_-) + O\bigl(\sqrt{\log n / n}\bigr),
\end{equation}
where $\bar{\bh}_c^{(1)}$ is the mean representation of class $c$.
The scaling factor $\lambda = (p-q)/(p+q)$ is negative iff $q > p$, in which case the aggregation reverses the direction of the class signal.
Its magnitude satisfies $|\lambda| < 1$ for all $p$ and $q$, it vanishes iff $p = q$, and it approaches $1$ as $q/p \to \infty$.
\end{theorem}

In other words, under strong heterophily the class signal is reversed but hardly attenuated, and the attenuation is strongest when $p$ and $q$ are close.
We refer to this operator as mean aggregation because, in the dense CSBM, all degrees concentrate around the same value, so the symmetric normalization agrees with a mean over the neighbors at leading order.

\begin{proof}
Let $\delta_n = \sqrt{\log n / n}$.
For a node $i$, let $n_s(i) \sim \mathrm{Bin}(n/2-1, p)$ and $n_d(i) \sim \mathrm{Bin}(n/2, q)$ be the numbers of its same-class and cross-class neighbors, so that its degree with the self-loop is $d_i = 1 + n_s(i) + n_d(i)$.
By Hoeffding's inequality and a union bound over the nodes, there is an event $\mathcal{E}$ of probability at least $1 - 4n^{-3}$ on which $|n_s(i) - np/2| \leq 2\sqrt{n \log n}$ and $|n_d(i) - nq/2| \leq 2\sqrt{n \log n}$ for all $i$.
On $\mathcal{E}$, every degree satisfies $d_i = \bar{d}\,(1 + O(\delta_n))$ with $\bar{d} = n(p+q)/2$, uniformly in $i$.

Fix a node $i \in V_+$.
Since the features are independent of the graph and $\boldsymbol{\mu}_- = -\boldsymbol{\mu}_+$, conditioning on the graph gives $\E[\bh_i^{(1)} \mid \G] = c_i\, \boldsymbol{\mu}_+$, where
\[
    c_i = \sum_{j \in \N(i),\, y_j = +1} \frac{1}{\sqrt{d_i d_j}} \;-\; \sum_{j \in \N(i),\, y_j = -1} \frac{1}{\sqrt{d_i d_j}}
\]
and $\N(i)$ includes $i$.
On $\mathcal{E}$, each term equals $(1 + O(\delta_n))/\bar{d}$, the first sum has $n_s(i) + 1$ terms, and the second has $n_d(i)$ terms, so $c_i = \lambda + O(\delta_n)$.
In all cases, $|c_i| \leq \sum_{j \in \N(i)} d_i^{-1/2} = \sqrt{d_i} \leq \sqrt{n}$, since $d_j \geq 1$.
Hence $\E[c_i] = \lambda + O(\delta_n)$, i.e.,
\begin{align}
    \E[\bh_i^{(1)} | y_i = +1] &= \frac{p}{p+q}\boldsymbol{\mu}_+ + \frac{q}{p+q}\boldsymbol{\mu}_- + O(\delta_n) = \frac{p-q}{p+q} \cdot \frac{\mu}{2}\mathbf{e}_1 + O(\delta_n).
\end{align}
By symmetry, the same holds for $y_i = -1$ with the opposite sign.
Averaging over the nodes of each class and subtracting gives the claim.
This completes the proof. \qed
\end{proof}

The next result shows that a sufficiently strong preference for same-class edges prevents the reversal.

\begin{theorem}[Cost-sensitive aggregation preserves signal direction]\label{thm:preservation}
Under the same CSBM, consider the weighted aggregation
\[
    \bh_i^{(1)} = \frac{\sum_{j \in \N(i)} w_{ij}\, \bx_j}{\sum_{j \in \N(i)} w_{ij}},
\]
where $\N(i)$ includes $i$ and every pair $(i,j)$ carries a weight $w_{ij} \in (0, W]$.
Suppose that, given the labels, the weights are mutually independent and are jointly independent of the edges and of the features, with $\E[w_{ij} | y_i = y_j] = w_+$ and $\E[w_{ij} | y_i \neq y_j] = w_-$, where $W$, $w_+$, and $w_-$ do not depend on $n$.
Then
\begin{equation}
    \E[\bar{\bh}_+^{(1)} - \bar{\bh}_-^{(1)}] = \frac{p\,w_+ - q\,w_-}{p\,w_+ + q\,w_-}(\boldsymbol{\mu}_+ - \boldsymbol{\mu}_-) + O\bigl(\sqrt{\log n / n}\bigr).
\end{equation}
The scaling factor is positive, i.e., the direction of the class signal is preserved, if and only if (iff) $w_+/w_- > q/p$.
\end{theorem}

In other words, the direction of the class signal is preserved exactly when the preference for same-class edges, measured by $w_+/w_-$, outweighs their relative scarcity, measured by $q/p$.
With equal weights, the scaling factor is that of \Cref{thm:degradation}.
For example, for $p = 0.1$ and $q = 0.4$, equal weights give $\lambda = -0.6$, whereas weights with $w_+/w_- = 8$, which is above $q/p = 4$, give the scaling factor $(0.8-0.4)/(0.8+0.4) = 1/3$.
Note that the theorem concerns the sign of the scaling factor and not its magnitude, which in this example is smaller than under mean aggregation.

\begin{proof}
Let $\delta_n = \sqrt{\log n / n}$ and $\kappa = (p\,w_+ - q\,w_-)/(p\,w_+ + q\,w_-)$.
Fix a node $i \in V_+$, and let $A_i$ and $B_i$ be the total weights of its neighbors in $V_+$ (including $i$ itself) and in $V_-$, respectively.
Since the features are independent of the pair formed by the graph and the weights, and $\boldsymbol{\mu}_- = -\boldsymbol{\mu}_+$, conditioning on the graph and on the weights gives
\[
    \E[\bh_i^{(1)} \mid \G, w] = \frac{A_i\, \boldsymbol{\mu}_+ + B_i\, \boldsymbol{\mu}_-}{A_i + B_i} = \frac{A_i - B_i}{A_i + B_i}\, \boldsymbol{\mu}_+ .
\]
Apart from the weight of the self-loop, which is at most $W$, $A_i$ is a sum of $n/2 - 1$ independent terms in $[0, W]$ with mean $p\,w_+$ (the weight of a pair times the indicator of its edge), and $B_i$ is a sum of $n/2$ such terms with mean $q\,w_-$.
By Hoeffding's inequality, with probability at least $1 - 4n^{-4}$ we have $A_i = n p\, w_+/2 + O(\sqrt{n \log n})$ and $B_i = n q\, w_-/2 + O(\sqrt{n \log n})$, and then $(A_i - B_i)/(A_i + B_i) = \kappa + O(\delta_n)$, because $p\,w_+ + q\,w_- > 0$.
Since $|A_i - B_i| \leq A_i + B_i$ in all cases, taking the expectation gives
\begin{align}
    \E[\bh_i^{(1)} | y_i = +1] &= \frac{p\, w_+}{p\, w_+ + q\, w_-}\boldsymbol{\mu}_+ + \frac{q\, w_-}{p\, w_+ + q\, w_-}\boldsymbol{\mu}_- + O(\delta_n).
\end{align}
The case $y_i = -1$ is symmetric, and averaging over the nodes of each class and subtracting gives the claim.
Finally, $\kappa > 0$ iff $p\,w_+ > q\,w_-$, i.e., iff $w_+/w_- > q/p$.
This completes the proof. \qed
\end{proof}

\begin{remark}[Applicability to CSNA]\label{rem:applicability}
With the normalization of \Cref{eq:con}, the concordant channel is a weighted aggregation of the form of \Cref{thm:preservation}, applied to the transformed representations, with the weights $w_{ij} = e^{s_{ij}}$.
Three caveats limit the use of the theorem.
First, these weights are functions of the features of the two endpoints, so the independence assumption of the theorem does not hold, and the theorem is only a heuristic guide for CSNA.
Second, the theorem is \emph{conditional}.
It shows what happens \emph{if} the weights achieve the separation, not that CSNA \emph{will} learn it.
In the lite variant, moreover, no term of the training objective rewards the separation directly.
Third, the design of the routing bounds the separation that the weights can express.
Since the concordance scores lie in $(0, 1/2]$, the weights $e^{s_{ij}}$ lie in $(1, e^{1/2}]$.
Hence $w_+/w_- < e^{1/2} \approx 1.65$, whatever the cost function and the temperature, and the condition $w_+/w_- > q/p$ cannot hold when $q/p$ exceeds this value.
With the per-source normalization of our implementation, the channel is not literally of this form, but the cost has the same limited effect, namely every routing weight is within a factor of $1.65$ of one over the degree of the sending node.
For the benchmarks, we estimate $q/p$ from the homophily ratio and the class sizes (\Cref{app:multiclass}).
The estimate is about $6.0$ on Texas, $2.7$ on Cornell, and $2.0$ on Wisconsin, all above the bound, so the condition is out of reach on these three datasets.
On Actor, Chameleon, and Squirrel it is between $0.85$ and $0.96$, i.e., below $1$, so for these datasets the model predicts no reversal even under uniform weights.
These estimates are indicative only, since the benchmark graphs are not draws from a CSBM.
\Cref{sec:experiments} measures the separation that is achieved in practice.
\end{remark}

\begin{remark}[Scope and limitations]\label{rem:scope}
(i)~The CSBM analysis is for $C=2$, whereas the experiments have $C=5$.
The qualitative conclusion carries over (\Cref{app:multiclass}), but the scaling factor changes, and for $q > p$ its magnitude is below $1/(C-1)$, so the signal is attenuated as well as reversed.
(ii)~\Cref{thm:preservation} concerns the expected difference between the class means, not classification accuracy.
Indeed, a sign-reversed but well-separated representation can be corrected by a downstream linear classifier.
Thus the theorems show that cost-sensitive weighting preserves the direction of the class signal where mean aggregation may reverse it, and they do not by themselves prove improved classification.
The theorems also do not control the within-class scatter, on which classification depends as well.
(iii)~The analysis covers a single layer.
Multi-layer interactions between the evolving representations and the cost function are not analyzed.
(iv)~The theorems are asymptotic and assume a dense graph (constant $p$ and $q$), whereas the benchmark graphs are sparse.
\end{remark}

\section{Experiments}\label{sec:experiments}

The theorems of \Cref{sec:theory} are conditional.
They do not say whether a learned cost function separates the edge types on real graphs, nor whether preserving the direction of the class signal improves accuracy.
We therefore evaluate CSNA empirically on node classification, and examine on which datasets per-edge routing helps.

\subsection{Datasets}

We evaluate on six heterophily benchmarks (\Cref{tab:datasets}).
Texas, Wisconsin, and Cornell are webpage graphs from WebKB~\cite{pei2020geom}, Chameleon and Squirrel are Wikipedia article networks~\cite{pei2020geom}, and Actor is a co-occurrence network from film databases~\cite{pei2020geom}.
Chameleon and Squirrel have known data-quality issues, namely duplicate nodes~\cite{platonov2023critical}.
We refer to Texas, Wisconsin, Cornell, and Actor as \emph{adversarial-heterophily} datasets, and to Chameleon and Squirrel as \emph{informative-heterophily} datasets.
We define this grouping operationally, by whether plain aggregation helps, i.e., GCN is less accurate than the MLP on the former and more accurate on the latter (\Cref{tab:results}).

\begin{table}[t]
\caption{Dataset statistics. $\mathcal{H}$ is the edge homophily ratio.}\label{tab:datasets}
\centering
\begin{tabular}{lccccc}
\toprule
Dataset & Nodes & Edges & Features & Classes & $\mathcal{H}$ \\
\midrule
Texas & 183 & 287 & 1,703 & 5 & 0.09 \\
Wisconsin & 251 & 458 & 1,703 & 5 & 0.19 \\
Cornell & 183 & 278 & 1,703 & 5 & 0.13 \\
Actor & 7,600 & 26,705 & 932 & 5 & 0.22 \\
Chameleon & 2,277 & 31,396 & 2,325 & 5 & 0.23 \\
Squirrel & 5,201 & 198,423 & 2,089 & 5 & 0.22 \\
\bottomrule
\end{tabular}
\end{table}

\subsection{Setup}

\paragraph{Baselines.}
We compare CSNA with seven methods, namely an MLP, which uses no graph structure, GCN~\cite{kipf2017semi}, GAT~\cite{velickovic2018graph}, GraphSAGE~\cite{hamilton2017inductive}, H2GCN~\cite{zhu2020beyond}, GPRGNN~\cite{chien2021adaptive}, and ACM-GNN~\cite{luan2022revisiting}.

\paragraph{Protocol.}
We used 10 random splits of each dataset into 60\% training, 20\% validation, and 20\% test nodes.
To ensure a fair comparison, \emph{all} methods were tuned over the same grid for the hyperparameters that they share, namely the learning rate and the hidden dimension.
For CSNA, we additionally tuned the temperature $\tau$.
All models used 2 layers, dropout 0.5, the Adam optimizer with weight decay $5 \times 10^{-4}$, and early stopping.
We saved the checkpoint with the best validation accuracy and evaluated it \emph{once} on the test set.
We report the classification accuracy, i.e., the percentage of correctly labeled test nodes, averaged over the 10 splits, with the standard deviation reflecting the split-to-split variability.
Here CSNA denotes the lite variant ($g_{ij}$ only), trained without edge sampling (\Cref{app:variants}).
\Cref{app:repro} gives the grid, the tuning protocol, the hardware, and the software.

\subsection{Main Results}

\Cref{tab:results} reports the accuracy of all methods, and we discuss it by regime.
Many differences between methods are within one standard deviation, so we focus on patterns across datasets and not on strict rankings of the means.

\begin{table}[t]
\caption{Node classification accuracy (\%) on the heterophily benchmarks, as the mean $\pm$ the standard deviation over 10 random splits. All methods were tuned over the same grid for their shared hyperparameters. The best mean per dataset is in \textbf{bold}, and the second best is \underline{underlined}.}\label{tab:results}
\centering
\setlength{\tabcolsep}{3.5pt}
\small
\begin{tabular}{lcccccc}
\toprule
Method & Texas & Wisconsin & Cornell & Actor & Chameleon & Squirrel \\
 & $\mathcal{H}$=0.09 & $\mathcal{H}$=0.19 & $\mathcal{H}$=0.13 & $\mathcal{H}$=0.22 & $\mathcal{H}$=0.23 & $\mathcal{H}$=0.22 \\
\midrule
MLP & 77.3{\tiny$\pm$4.6} & \textbf{83.7{\tiny$\pm$4.8}} & \underline{72.2{\tiny$\pm$3.6}} & 35.0{\tiny$\pm$1.4} & 51.9{\tiny$\pm$1.8} & 34.8{\tiny$\pm$1.4} \\
GCN & 55.7{\tiny$\pm$9.9} & 50.6{\tiny$\pm$8.5} & 47.0{\tiny$\pm$8.7} & 27.3{\tiny$\pm$1.4} & \textbf{67.3{\tiny$\pm$1.7}} & \textbf{53.4{\tiny$\pm$0.8}} \\
GAT & 50.8{\tiny$\pm$9.8} & 51.4{\tiny$\pm$7.8} & 47.3{\tiny$\pm$5.8} & 28.0{\tiny$\pm$1.4} & 65.7{\tiny$\pm$2.2} & 50.4{\tiny$\pm$1.4} \\
GraphSAGE & 76.5{\tiny$\pm$6.8} & 75.9{\tiny$\pm$5.8} & 66.2{\tiny$\pm$7.7} & 34.1{\tiny$\pm$0.6} & 63.9{\tiny$\pm$2.1} & 45.8{\tiny$\pm$1.4} \\
H2GCN & \textbf{81.9{\tiny$\pm$4.2}} & \underline{82.4{\tiny$\pm$5.8}} & \underline{72.2{\tiny$\pm$4.2}} & 35.6{\tiny$\pm$0.9} & 56.8{\tiny$\pm$2.7} & 35.0{\tiny$\pm$1.3} \\
GPRGNN & \underline{77.8{\tiny$\pm$6.9}} & 75.7{\tiny$\pm$7.7} & 58.6{\tiny$\pm$9.9} & \textbf{36.0{\tiny$\pm$0.8}} & 65.0{\tiny$\pm$2.1} & 43.6{\tiny$\pm$1.8} \\
ACM-GNN & 77.3{\tiny$\pm$8.2} & 78.0{\tiny$\pm$4.4} & 68.6{\tiny$\pm$7.2} & 35.3{\tiny$\pm$1.2} & \underline{65.9{\tiny$\pm$2.9}} & \underline{51.0{\tiny$\pm$1.8}} \\
\midrule
CSNA (ours) & 77.0{\tiny$\pm$8.3} & 79.6{\tiny$\pm$6.2} & \textbf{72.7{\tiny$\pm$4.3}} & \underline{35.7{\tiny$\pm$1.2}} & 54.6{\tiny$\pm$2.7} & 37.8{\tiny$\pm$1.9} \\
\bottomrule
\end{tabular}
\end{table}

\paragraph{Adversarial Heterophily (Texas, Wisconsin, Cornell, Actor).}
On these datasets, plain aggregation of neighbor features degrades classification.
The MLP is a strong baseline, and the standard GNNs GCN and GAT are roughly 7 to 33pp below it.
H2GCN is the most accurate graph-based method on Texas and Wisconsin.
CSNA has the highest mean on Cornell and the second highest on Actor, with gaps of at most 0.5pp to the nearest methods in both cases, which is well within one standard deviation.
On Texas and Wisconsin, it is 4.9pp and 4.1pp below the most accurate method (H2GCN and the MLP, respectively), and within one standard deviation of the other heterophily-aware methods.
Note that CSNA exceeds the MLP by at most 0.7pp on these four datasets (on Actor), so its accuracy here does not show a benefit of using the graph.

\paragraph{Informative Heterophily (Chameleon, Squirrel).}
On these datasets, the heterophilous structure itself carries useful signal.
GCN is the most accurate method, and ACM-GNN is 1.4pp and 2.4pp below it.
In contrast, CSNA and H2GCN are roughly 11 to 18pp below GCN.
Here the per-edge routing of CSNA does not help, and methods that apply uniform spectral filters (ACM-GNN and GPRGNN) or standard aggregation (GCN) are more accurate.
That heterophily can either harm or help aggregation is increasingly recognized~\cite{luan2022revisiting,platonov2023critical}, and our results agree that no single approach dominates both regimes.

\paragraph{Comparison with ACM-GNN.}
Recall that ACM-GNN and CSNA differ in the granularity of the routing.
On the adversarial-heterophily datasets, their accuracies are comparable, with means that differ by 0.3 to 4.1pp, in favor of CSNA on three of the four datasets.
On the informative-heterophily datasets, ACM-GNN is more accurate than CSNA by 11.3pp on Chameleon and 13.2pp on Squirrel (\Cref{sec:discussion} discusses a possible explanation).

\paragraph{Does the Learned Cost Separate the Edge Types?}
The hypothesis of \Cref{sec:intro} and the condition of \Cref{thm:preservation} both require that the cost function assigns lower costs to same-class edges than to cross-class edges.
To test this, we use the first-layer concordance score $s_{ij}$ as a binary classifier of same-class against cross-class edges.
We report its area under the ROC curve (AUC), which is 0.5 for random scores (\Cref{tab:auc}).
Self-loops are excluded, since they have zero cost and are same-class by construction.
Including them raises the AUC to about 0.9 on the three WebKB graphs, where self-loops outnumber the same-class edges, and the first version of this paper (arXiv v1) reported values that included them.

\begin{table}[t]
\caption{AUC of the first-layer concordance score $s_{ij}$ as a classifier of same-class against cross-class edges, for the lite variant, as the mean $\pm$ the standard deviation over 10 splits. The last column includes the self-loops and is shown only to document their effect.}\label{tab:auc}
\centering
\small
\begin{tabular}{llcc}
\toprule
Dataset & Regime & AUC & AUC with Self-Loops \\
\midrule
Texas & Adversarial & 0.63{\tiny$\pm$0.12} & 0.95 \\
Wisconsin & Adversarial & 0.73{\tiny$\pm$0.07} & 0.90 \\
Cornell & Adversarial & 0.57{\tiny$\pm$0.09} & 0.89 \\
Actor & Adversarial & 0.52{\tiny$\pm$0.01} & 0.71 \\
Chameleon & Informative & 0.57{\tiny$\pm$0.02} & 0.63 \\
Squirrel & Informative & 0.52{\tiny$\pm$0.03} & 0.54 \\
\bottomrule
\end{tabular}
\end{table}

The separation is weak.
The AUC exceeds 0.6 only on Texas and Wisconsin (0.63 and 0.73), where it also varies strongly across the runs (these graphs have few same-class edges, so the AUC rests on few pairs).
It is 0.57 on both Cornell and Chameleon, which belong to different regimes, and 0.52 on both Actor and Squirrel, which also belong to different regimes.
Thus the separation does not distinguish the two regimes, and it cannot serve as a diagnostic of the regime, as the first version of this paper proposed.
As a rough indication of the effect on the weights, let $s_+$ and $s_-$ denote the mean scores of same-class and cross-class edges.
Two edges with these scores in the same softmax receive weights whose ratio is $e^{s_+ - s_-}$, which is at most 1.10 on all datasets, far below the estimates of $q/p$ on the three WebKB graphs (\Cref{rem:applicability}).

The accuracy of CSNA on the adversarial-heterophily datasets is therefore unlikely to be explained by routing, although we did not run an ablation that disables the routing.
Its gains over GCN on these datasets (roughly 8 to 29pp) are close to those of the MLP over GCN (roughly 8 to 33pp), which suggests that they come mainly from the ego channel.
On average, the gate weights are consistent with this reading, with strong variation across splits (\Cref{app:gate}).

\paragraph{Variants.}
\Cref{app:variants} compares the lite and the extended variant, each trained with and without stochastic edge sampling~\cite{rong2020dropedge}.
The comparison motivates presenting the lite variant without sampling as the default.

\section{Limitations}\label{sec:discussion}

We discuss five limitations of CSNA and of our study.

\paragraph{Homophilous Graphs.}
CSNA is less accurate than GCN by roughly 4 to 11pp on Cora, CiteSeer, and PubMed (\Cref{app:homophily}).
This is expected, since the cost function has little to separate when most edges are already same-class, so per-edge routing is not justified when standard aggregation suffices.

\paragraph{Informative Heterophily.}
We have no tested explanation for the gap between CSNA and both GCN and ACM-GNN on Chameleon and Squirrel.
Since the routing weights are close to uniform (see the next paragraph), the gap is unlikely to come from the weights themselves.
It may come from the separate transformations of the two channels or from the gate.

\paragraph{Weak Routing.}
The routing that our experiments evaluate is weak.
The learned cost separates the edge types only weakly (\Cref{tab:auc}), and every routing weight is within a factor of $1.65$ of the uniform weight (\Cref{rem:applicability}).
Hence our negative findings concern this form of per-edge routing, and they do not rule out that a sharper form would help.

\paragraph{Scope of the Evaluation.}
Our experiments include no ablation that isolates the routing, for example CSNA with uniform routing weights, so our conclusions about routing rest on the edge-separation analysis and on comparisons across architectures.
Some baselines are less accurate here than in the literature, which may reflect our small tuning grid and our random splits, and Chameleon and Squirrel are used unfiltered despite their duplicate nodes~\cite{platonov2023critical}.

\paragraph{Runtime.}
Training CSNA is $3$--$9\times$ slower than training GCN, and at equal hidden dimension CSNA has more parameters (\Cref{app:runtime}).

\section{Conclusions}\label{sec:conclusion}

We considered the question of when per-edge routing of messages helps on heterophilous graphs.
We studied it with CSNA, a GNN layer that soft-routes each message through a concordant and a discordant channel according to a learned distance between the endpoints of its edge.
We showed, under a two-class CSBM, that a weighted aggregation with feature-independent weights preserves the direction of the class signal, which mean aggregation may reverse, iff $w_+/w_- > q/p$.
In our experiments, the weights of CSNA barely distinguished the edge types.
The learned cost separated same-class from cross-class edges only weakly (an AUC of 0.52 to 0.73), with no consistent difference between the regimes, and the design keeps every routing weight within a factor of $1.65$ of the uniform weight.
CSNA exceeded an MLP by at most 0.7pp on the adversarial-heterophily datasets, and it was 11 to 16pp below GCN and ACM-GNN on the informative-heterophily datasets.
Thus, on these benchmarks, we found no evidence that per-edge routing in this form helps.
It would be interesting to test sharper routing, for example a softmax over the negated costs divided by the temperature, which removes the bound on the ratio of the weights, together with a training signal that rewards the separation of the edge types.
It would also be interesting to adapt per-edge routing to the informative-heterophily regime, possibly by learning to leverage cross-class signal instead of suppressing it.

\bibliographystyle{splncs04}
\bibliography{references}

\appendix
\crefalias{section}{appendix}
\section*{Appendix}
\addcontentsline{toc}{section}{Appendix}

\section{CSNA Variant Comparison}\label{app:variants}

We compare four CSNA variants in a $2 \times 2$ design, namely the extended variant ($g+h$) and the lite variant ($g$ only), each trained with and without stochastic edge sampling~\cite{rong2020dropedge}.
Edge sampling randomly removes edges in each training epoch and is a natural option for scaling CSNA to larger graphs.
\Cref{tab:variants} reports the results.

\begin{table}[h]
\caption{Accuracy (\%) of the four CSNA variants on the six datasets. The extended variant uses both $g_{ij}$ and $h_{ij}$, and the lite variant uses $g_{ij}$ only. Sampling denotes edge sampling during training. The best variant per dataset is in \textbf{bold}.}\label{tab:variants}
\centering
\small
\begin{tabular}{lcccccc}
\toprule
Variant & Texas & Wisconsin & Cornell & Actor & Chameleon & Squirrel \\
\midrule
Extended, no sampling & \textbf{77.0} & \textbf{84.1} & 70.8 & 35.6 & 53.0 & \textbf{38.2} \\
Extended, sampling & 74.1 & 79.8 & 70.5 & \textbf{36.3} & 51.8 & 37.4 \\
Lite, no sampling & \textbf{77.0} & 79.6 & \textbf{72.7} & 35.7 & \textbf{54.6} & 37.8 \\
Lite, sampling & 75.7 & 79.6 & 70.3 & 35.8 & 53.5 & 37.0 \\
\bottomrule
\end{tabular}
\end{table}

\paragraph{Discussion.}
Without sampling, the extended variant is more accurate than the lite variant by more than 2pp only on Wisconsin (84.1 against 79.6), where the learned component $h_{ij}$ provides a useful correction beyond the distance $g_{ij}$.
On the other datasets, the lite variant is at most 0.4pp below the extended variant and up to 1.9pp above it, which suggests that the distance $g_{ij}$ alone is a sufficient cost for routing.
Since $h_{ij}$ introduces additional learnable parameters without a consistent advantage, it risks overfitting, particularly on the small datasets in our benchmark suite.
We therefore use the lite variant ($g_{ij}$ only) as the default.

Edge sampling lowers the accuracy in 9 of the 12 cases, by 0.3 to 2.9pp in eight of them and by 4.3pp for the extended variant on Wisconsin.
It leaves the accuracy of the lite variant on Wisconsin unchanged, and it raises the accuracy on Actor (by 0.7pp for the extended variant and 0.1pp for the lite variant), where it may act as a regularizer.
The accuracy loss due to sampling is modest, which makes it a viable strategy for scaling to larger graphs.

\section{Homophily Benchmarks}\label{app:homophily}

To examine the behavior of CSNA across the homophily spectrum, we evaluated it on three standard homophilous datasets (\Cref{tab:homophily}).

\begin{table}[h]
\caption{Accuracy (\%) on homophilous benchmarks. CSNA is less accurate than GCN by roughly 4 to 11pp.}\label{tab:homophily}
\centering
\begin{tabular}{lccc}
\toprule
Method & Cora & CiteSeer & PubMed \\
\midrule
GCN & \textbf{77.9} & \textbf{65.7} & \textbf{76.0} \\
CSNA & 67.0 & 56.3 & 72.2 \\
\bottomrule
\end{tabular}
\end{table}

CSNA is less accurate than GCN by 10.9pp on Cora, 9.4pp on CiteSeer, and 3.8pp on PubMed.
This is expected.
On homophilous graphs, most edges are already same-class, so standard aggregation is effective and the cost function cannot improve on uniform aggregation.
Thus the overhead of per-edge routing is not justified on such graphs.

\section{Gate Weight Analysis}\label{app:gate}

\Cref{tab:gate_weights} reports the gate weights of the lite variant at both CSNA layers, averaged over the nodes, as the mean $\pm$ the standard deviation over the 10 runs of \Cref{tab:auc}.

\begin{table}[h]
\caption{Gate weights $\gamma$ of the lite variant at the two CSNA layers, averaged over the nodes. Entries are the mean $\pm$ the standard deviation over 10 splits.}\label{tab:gate_weights}
\centering
\small
\begin{tabular}{lcccc}
\toprule
Dataset & Layer & $\gamma_{\text{con}}$ & $\gamma_{\text{dis}}$ & $\gamma_{\text{self}}$ \\
\midrule
Texas & 1 & 0.20{\tiny$\pm$0.29} & 0.18{\tiny$\pm$0.24} & 0.62{\tiny$\pm$0.31} \\
 & 2 & 0.03{\tiny$\pm$0.06} & 0.06{\tiny$\pm$0.08} & 0.91{\tiny$\pm$0.13} \\
Wisconsin & 1 & 0.18{\tiny$\pm$0.30} & 0.20{\tiny$\pm$0.37} & 0.62{\tiny$\pm$0.40} \\
 & 2 & 0.10{\tiny$\pm$0.16} & 0.26{\tiny$\pm$0.38} & 0.64{\tiny$\pm$0.43} \\
Cornell & 1 & 0.28{\tiny$\pm$0.29} & 0.35{\tiny$\pm$0.37} & 0.38{\tiny$\pm$0.33} \\
 & 2 & 0.02{\tiny$\pm$0.03} & 0.12{\tiny$\pm$0.22} & 0.86{\tiny$\pm$0.24} \\
Actor & 1 & 0.11{\tiny$\pm$0.11} & 0.17{\tiny$\pm$0.17} & 0.73{\tiny$\pm$0.27} \\
 & 2 & 0.27{\tiny$\pm$0.32} & 0.10{\tiny$\pm$0.13} & 0.63{\tiny$\pm$0.39} \\
Chameleon & 1 & 0.25{\tiny$\pm$0.26} & 0.33{\tiny$\pm$0.29} & 0.43{\tiny$\pm$0.33} \\
 & 2 & 0.41{\tiny$\pm$0.38} & 0.06{\tiny$\pm$0.11} & 0.52{\tiny$\pm$0.34} \\
Squirrel & 1 & 0.23{\tiny$\pm$0.36} & 0.29{\tiny$\pm$0.21} & 0.49{\tiny$\pm$0.30} \\
 & 2 & 0.42{\tiny$\pm$0.42} & 0.12{\tiny$\pm$0.15} & 0.46{\tiny$\pm$0.43} \\
\bottomrule
\end{tabular}
\end{table}

\paragraph{Variation Across Splits.}
The standard deviations across splits are about as large as the means.
For example, on Wisconsin the first-layer weight of the ego channel exceeds 0.9 in five of the ten splits and is below 0.1 in three.
Thus the gate weights of a single run do not characterize a dataset.
Note also that the first-layer means on Texas and Wisconsin are close to the values at initialization, $\mathrm{softmax}([0, 0, 1]) \approx (0.21, 0.21, 0.58)$, although the individual runs are far from these values.
An earlier version of this paper described three dataset-specific strategies on the basis of a single run of the extended variant, and the present runs do not support that description.

\paragraph{The Ego Channel.}
On the four adversarial-heterophily datasets, the mean weight of the ego channel is at least 0.6 in seven of the eight rows of \Cref{tab:gate_weights}, the exception being the first layer on Cornell (0.38).
This is consistent with the accuracy of CSNA on these datasets, which is within 0.7pp of that of the MLP on three of them and 4.1pp below it on Wisconsin (\Cref{tab:results}).
On Chameleon and Squirrel, the mean weight of the ego channel is 0.43 to 0.52, so the two aggregation channels together receive about half of the weight.
CSNA is nevertheless less accurate than GCN on these datasets, which suggests that its aggregation channels do not capture the useful heterophilous signal as effectively as the uniform aggregation of GCN.

\section{Runtime and Parameter Comparison}\label{app:runtime}

\Cref{tab:runtime} reports the training time and the number of parameters of each method.

\begin{table}[h]
\caption{Training time (seconds per split) and number of parameters (on Texas). All runs used the laptop CPU described in \Cref{app:repro}.}\label{tab:runtime}
\centering
\small
\begin{tabular}{lcccc|c}
\toprule
Method & Texas & Actor & Chameleon & Squirrel & Params (Texas) \\
\midrule
MLP & 0.2 & 3 & 1 & 2 & 219K \\
GCN & 0.2 & 4 & 7 & 62 & 219K \\
GAT & 0.3 & 11 & 20 & 152 & 110K \\
GraphSAGE & 0.3 & 15 & 30 & 132 & 219K \\
H2GCN & 0.6 & 11 & 27 & 327 & 318K \\
GPRGNN & 0.6 & 9 & 14 & 111 & 219K \\
CSNA & 1 & 34 & 37 & 205 & 144K \\
\bottomrule
\end{tabular}
\end{table}

Training CSNA is $3$--$9\times$ slower than training GCN, due to the per-edge distance computation.
On the largest dataset (Squirrel, 198K edges), a single split takes about 205 seconds, against 62 for GCN.
The parameter counts in \Cref{tab:runtime} depend on the hidden dimension selected for each method, which differs between the methods.
At equal hidden dimension, CSNA has more parameters than GCN, namely 144K against 110K at dimension 64 and 353K against 219K at dimension 128 (on Texas).
For scaling to larger graphs, edge sampling~\cite{rong2020dropedge} is a natural option.
In our comparison of variants (\Cref{app:variants}), it lowers the accuracy by at most 3pp in all cases but one.

\section{Multi-Class Extension}\label{app:multiclass}

The binary analysis extends to $C > 2$ classes at leading order.
Consider a CSBM with $C$ equal-sized classes with means $\boldsymbol{\mu}_1, \ldots, \boldsymbol{\mu}_C$, intra-class probability $p$, and uniform inter-class probability $q$, and let $\mathbf{m} = \sum_{c} \boldsymbol{\mu}_c$.
Then the following hold.
\begin{itemize}
    \item The fraction of same-class neighbors of a node tends to $p/(p + (C-1)q)$, which we identify with $\mathcal{H}$. Hence $q/p = (1-\mathcal{H})/((C-1)\,\mathcal{H})$. For classes of unequal sizes, with fractions $\pi_1, \ldots, \pi_C$ of the nodes, counting the same-class and cross-class pairs gives instead $q/p = \frac{1-\mathcal{H}}{\mathcal{H}} \cdot \frac{S}{1-S}$ with $S = \sum_c \pi_c^2$, which is the estimate used in \Cref{rem:applicability}.
    \item Under mean aggregation, the expected representation of a node of class $c$ is, at leading order, $\bigl(p\,\boldsymbol{\mu}_c + q \sum_{c' \neq c} \boldsymbol{\mu}_{c'}\bigr)/(p + (C-1)q) = \bigl((p-q)\,\boldsymbol{\mu}_c + q\,\mathbf{m}\bigr)/(p + (C-1)q)$. Hence each pairwise class-mean difference is scaled by the signed factor $(p - q)/(p + (C-1)q)$, which is negative when $q > p$ (the direction of the class signal is reversed), and then has magnitude below $1/(C-1)$.
    \item Under the weighted aggregation of \Cref{thm:preservation}, the same computation gives the factor $(p\,w_+ - q\,w_-)/(p\,w_+ + (C-1)q\,w_-)$, which is positive iff $w_+/w_- > q/p$.
\end{itemize}
Thus the qualitative conclusion is unchanged, namely a sufficient preference for same-class edges preserves the direction of the class signal where mean aggregation may reverse it.

\section{Reproducibility}\label{app:repro}

\paragraph{Architecture.}
We apply an input MLP before the first CSNA layer and add residual connections.
We initialize the gate bias to $[0, 0, 1]$, which favors the ego channel, so that the model starts near MLP-like behavior.
Self-loops are added before the cost computation.
They receive $g_{ii}=0$ and hence the highest score, so the message of the node to itself has the largest weight in the concordant channel and the smallest in the discordant channel.
The softmax normalization of the concordance scores is performed per source node, i.e., over the edges that leave the sending node.
We also tested per-destination normalization, which is more common in message-passing GNNs and is the form shown in \Cref{eq:con,eq:dis}, and found no consistent difference across datasets.

\paragraph{Hardware.}
All experiments ran on a laptop with an Intel Core i7-1165G7 CPU (4 cores, 8 threads, 2.8\,GHz) and 32GB of RAM, under Linux Mint 20.3.
No GPU was used.

\paragraph{Software.}
The implementation uses Python, PyTorch, and PyTorch Geometric.
The source code is available at \url{https://github.com/eyal-weiss/CSNA-public}.

\paragraph{Splits.}
We used 10 random splits per dataset into 60\% training, 20\% validation, and 20\% test nodes, generated with seed 42.

\paragraph{Training.}
We used early stopping with patience 50 and at most 300 epochs.
We saved the checkpoint with the best validation accuracy and evaluated it once on the test set.
The lite variant was trained with the cross-entropy loss only, and the extended variant with the regularizer of \Cref{eq:cal}.

\paragraph{Edge-Separation Runs.}
The AUC values of \Cref{tab:auc} and the gate weights of \Cref{tab:gate_weights} come from a separate set of runs of the lite variant, with the configurations selected for \Cref{tab:results}, the same 10 splits, and the random seed $42+k$ for split $k$.
The mean accuracies of these runs are within 1.4pp of \Cref{tab:results}, except on Wisconsin, where they are 3.5pp higher (83.1 against 79.6).

\paragraph{Tuning Protocol.}
All methods were tuned over the same grid, as follows.
\begin{itemize}
    \item Learning rate: $\{0.01, 0.005\}$
    \item Hidden dimension: $\{64, 128\}$ on small datasets (Texas, Wisconsin, Cornell); $\{64\}$ on large datasets (Actor, Chameleon, Squirrel), for computational efficiency
    \item CSNA additionally: $\tau \in \{0.1, 0.5, 1.0, 2.0\}$
\end{itemize}
The best configuration was selected by the mean validation accuracy over 3 tuning splits (small datasets) or 2 tuning splits (large datasets), with the tuning epochs capped at 200 (small) or 150 (large).
The final results are reported on all 10 splits, with the full budget of 300 epochs and early stopping.

\end{document}